\documentclass{article}

\usepackage[english]{babel}
\usepackage[utf8x]{inputenc}
\usepackage[T1]{fontenc}
\usepackage{amsfonts}

\usepackage[margin = 1in]{geometry}
\usepackage{comment}
\usepackage[linesnumbered,ruled]{algorithm2e}
\usepackage[colorlinks=true,citecolor=blue]{hyperref}

\usepackage{color}

\usepackage{graphicx}
\usepackage{amssymb}
\usepackage{amsmath}
\usepackage{amsthm}
\usepackage{url}
\usepackage{float}
\usepackage{caption}
\usepackage{graphicx}
\usepackage{todonotes}
\newtheorem {theorem}{Theorem}
\newtheorem {lemma}[theorem]{Lemma}

\newtheorem {definition}[theorem]{Definition}

\newcommand{\paren}[1]{\ensuremath{ \left( {#1} \right) }}
\newcommand{\set}[1]{\ensuremath{\left\{#1\right\}}}

\renewcommand{\log}{\operatorname{log}}

\newcommand{\R}{\ensuremath{\mathbb{R}}}
\newcommand{\Z}{{\ensuremath{\mathbb{Z}}}}

\newcommand{\E}{\mathbb{E}}  
\newcommand{\X}{\mathcal{X}}
\newcommand{\F}{\mathcal{F}}

\renewcommand{\dim}{\ensuremath{n}}  

\newcommand{\alg}{\mathcal{A}}

\newcommand{\D}{\mathcal{D}}
\newcommand{\N}{\mathcal{N}}
\newcommand{\HV}{\mathcal{HV}}
\newcommand{\Sphere}{\mathcal{S}_+}
\newcommand{\GP}{\mathcal{GP}}

\title{Random Hypervolume Scalarizations for Provable Multi-Objective Black Box Optimization}

\author{
Daniel Golovin
\and Qiuyi (Richard) Zhang 
}
\date{Google Brain \thanks{\texttt{\{dgg,qiuyiz\}@google.com}}}

\begin{document}
\maketitle

\begin{abstract}
Single-objective black box optimization (also known as zeroth-order
optimization) is the process of minimizing a scalar objective $f(x)$, given evaluations at adaptively chosen inputs $x$. In this paper, we
consider multi-objective optimization, where $f(x)$ outputs a vector of
possibly competing objectives and the goal is to converge to the Pareto frontier. Quantitatively, we wish to maximize the standard \emph{hypervolume indicator} metric, which measures the dominated hypervolume of the entire set of chosen inputs. In this paper, we introduce a novel scalarization function, which we term the \emph{hypervolume scalarization}, and show that drawing random scalarizations from an appropriately chosen distribution can be used to efficiently approximate the \emph{hypervolume indicator} metric. We utilize this connection to show that Bayesian optimization with our scalarization via common acquisition functions, such as Thompson Sampling or Upper Confidence Bound, provably converges to the whole Pareto frontier by deriving tight \emph{hypervolume regret} bounds on the order of $\widetilde{O}(\sqrt{T})$. Furthermore, we highlight the general utility of our scalarization framework by showing that any provably convergent single-objective optimization process can be effortlessly converted to a multi-objective optimization process with provable convergence guarantees.
\end{abstract}

\section{Introduction}
Single-objective optimization has traditionally been the focus in the field of
machine learning for many practical and interesting applications, from standard
regression objectives \cite{kutner2005applied} to ever-increasingly complicated
losses used in deep learning \cite{lecun2015deep} and reinforcement learning
\cite{sutton1998introduction}. However, in the recent years, there has been a
growing need to care about multiple objectives in learning and understanding the
inherent tradeoffs between these conflicting objectives. Examples include
classical tradeoffs such as bias vs variance \cite{neal2018modern}, and accuracy
vs calibration \cite{guo2017calibration} but there are increasingly complex
tradeoffs between accuracy and robustness to attack
\cite{zhang2019theoretically}, accuracy and fairness
\cite{zliobaite2015relation}, between multiple correlated tasks in multi-task
learning \cite{kendall2018multi}, between network adaptations in meta-learning
\cite{finn2017model}, or any combination of the above.

To understand and visualize a growing number of complex tradeoffs, many turn to multi-objective optimization, which is the maximization of $k$ objectives
$F(x) := (f_1(x), ... , f_k(x))$ over the parameter space
$x\in \X \subseteq \R^n$. Since one often cannot simultaneously maximize all $f_i$, multi-objective
optimization aims to find the entire \emph{Pareto frontier} $\F$ of the objective
space, where intuitively $x$ is on the Pareto frontier if there is no way to
improve on all objectives simultaneously. To measure progress, a natural and widely used metric
to compare Pareto sets is the \emph{hypervolume indicator}, which is the volume of the
dominated portion of the Pareto set \cite{zitzler1999multiobjective}. The
hypervolume metric is especially desirable because it has strict Pareto
compliance meaning that if a set $A \subseteq B \subset \R^k$, then the
hypervolume of $B$ is greater than that of $A$. While the hypervolume indicator satisfies strict Pareto compliance, almost all other unary metrics do not, thereby explaining the popularity of using hypervolume as the predominant measure of progress
\cite{zitzler2003performance}.

While the hypervolume indicator remains the gold standard in evaluating
multi-objective algorithms, note that it is a function of a set of
objective vectors and not just a single evaluation. This means evolutionary algorithms often cannot use
the hypervolume as a fitness score although some discretized version exist
\cite{emmerich2005emo}. Furthermore, it is computationally inefficient: the best
current asymptotic runtime calculating the hypervolume indicator in $\R^k$ is
$O(n^{k/2})$ via a reduction to Klee's Measure Problem
\cite{beume2006faster}. This cannot be substantially sped up as hypervolume
calculation is shown to be {\bf $\#$P}-hard  
and even approximating the hypervolume improvement by a $O(2^{d^{1-\epsilon}})$ factor is {\bf
  NP}-hard for any $\epsilon > 0$ \cite{bringmann2010approximating}. For $k \leq 3$, fast albeit
complicated algorithms exist \cite{yang2019efficient}.

Due to the difficulty of using the hypervolume directly, many multi-objective
optimization problems use a heuristic-based scalarization strategy, which splits
the multi-objective optimization into numerous single "scalarized" objectives
\cite{roijers2013survey}. For some weights $\lambda \in \R^k$, we have
scalarization functions $s_\lambda(y) : \R^k \to \R$ that convert
multi-objective outputs into a single-objective scalars. Bayesian optimization
is then applied to this family of single-objective functions $s_\lambda(F(x))$
for various $\lambda$ and if we construct $s_\lambda$ to be monotonically
increasing in all coordinates, then $$x_\lambda = \arg\max_{x \in \X}
s_\lambda(F(x))$$ is on the Pareto frontier \cite{paria2018flexible}.

Early works on scalarization include heuristic-based algorithms such as ParEgo
\cite{knowles2006parego} and MOEAD \cite{zhang2007moea}. Even then, the most
popular scalarizations are the linear scalarization $s_\lambda(y) = \sum_i
\lambda_i y_i$ and the Chebyshev scalarization $s_\lambda(y) =\min_{i} \lambda_i
(y_i - z_i)$ for some reference point $z$ and some distribution over $\lambda$
\cite{nakayama2009sequential}. However, the choice of weight distribution,
reference point, and even the scalarization functions themselves is diverse and
largely various between different papers \cite{paria2018flexible,
  nakayama2009sequential, zhang2007moea}. Recently, some works have come up with
novel scalarizations that perform better empirically \cite{aliano2019exact,
  schmidt2019min} and others have tried to do comparisons between different
scalarizations with varying conclusions \cite{kasimbeyli2019comparison}. Some
have also proposed adaptively weighted approaches that have connections to
gradient-based multi-objective optimization \cite{NIPS2019_9374} .

In addition to scalarization, there are a large diverse array of multi-objective
optimization rules that are often heuristic-based and lack theoretical
guarantees, such as aggregation-based, decomposition-based, diversity-based,
elitism-based, gradient-based, and hybrid approaches \cite{emmerich2018tutorial,
  zitzler2004tutorial, konak2006multi}. Furthermore, many Bayesian optimization approaches have been extended to the multi-objective setting, such as predictive entropy search \cite{hernandez2016predictive} or uncertainty measures \cite{picheny2015multiobjective}.  To our knowledge, there have been
limited theoretical works that give regret and convergence bounds for
multi-objective optimization. Some works show that the algorithms achieve small
Pareto regret, which only guarantees that one recovers a single point close to the
Pareto frontier \cite{lu2019multi, oner2018combinatorial,turugay2018multi}. For
many scalarization functions, Paria et al provides a Bayes regret bound with
respect to a scalarization-induced regret \cite{paria2018flexible} but small Bayes regret with a monotone scalarization only guarantees recovery of a single Pareto optimal point . Zuluaga et
al provides sub-linear hypervolume regret bounds; however, they are exponential
in $k$ and its analysis only applies to a specially tailored algorithm that
requires an unrealistic classification step \cite{zuluaga2013active}.

\subsection{Our contributions}

In this paper, we introduce a new scalarization function, which we term the hypervolume scalarization, and present a novel connection between this scalarization function and the hypervolume indicator. We provide the first hypervolume regret bounds that are valid for common Bayesian optimization algorithms such as UCB and Thompson Sampling and can be polynomial in the number of objectives $k$ when $\X$ is suitably compact. This implies that our choice of scalarization functions and weight distribution $\D_\lambda$ is theoretically sound and leads to provable convergence to the whole Pareto frontier. Furthermore, we can utilize the connection of scalarization to hypervolume indicator to provide a computationally-efficient and accurate estimator of the hypervolume, providing a simple algorithm for approximating the hypervolume that easily generalizes to higher dimensions. The main lemma that we rely on in this paper is the following:

\begin{lemma}[Hypervolume Scalarization: Informal Restatement of Lemma~\ref{lem:hypervolume}]
  Let $Y = \set{y_1,.., y_m}$ be a set of $m$ points in $\R^k$ and let $z \in
  \R^k$ be a reference point such that $y_i \geq z$ for all $i$. Then, the
  hypervolume $\HV_z(Y)$ of $Y$ with respect to $z$ is given by:
  $$\HV_z(Y) := c_k \E_{\lambda \sim \D} \left [ \max_{y \in Y} s_\lambda(y - z) \right ]$$
  for some scalarization functions $s_\lambda(y)$ and fixed weight distribution $\D$.
\end{lemma}

Note our resulting hypervolume scalarization functions are similar to the
Chebyshev scalarization approach $s_\lambda(y) = \underset{i}{\min} \lambda_i
(y_i - z_i)$ but they come with provable guarantees. Specifcally, using these scalarizations, we can first easily extend a single-objective optimization algorithm into the multi-objective setting: at each step, we first apply a
random scalarization, then we attempt to maximize along that scalarization
function via any single-objective optimization procedure to find the
next suggested point. We show novel regret bounds for combining our hypervolume
scalarization with popular and standard Bayesian optimization algorithms, such as Thompson Sampling
(TS) and Upper Confidence Bound (UCB) methods.

If we let $x_t \in \X$ denote the point chosen in step $t$, and let
$Y_t := \set{F(x_i) : i = 1, 2, \ldots, t}$, then the \emph{hypervolume regret} at
time $t$ is defined as $r_t := \HV_z(Y^*) - \HV_z(Y_t)$, where $Y^*$ is the Pareto frontier.
We show that our cumulative hypervolume
regret at step $T$ for TS and UCB is bounded by $O(k^2 \sqrt{\gamma_T T \ln T})$, where $\gamma_T$ is a kernel-dependent quantity known as the maximum information gain. Intuitively, the maximum information gain allows us to quantify the increase in certainty after evaluating $T$ points and is usually mild. For example, $\gamma_T = O(\text{poly}\log(T))$ for the squared exponential kernel and since $r_t$ is by
definition monotonically decreasing, our regret bound translates immediately
into a convergence bound of $\widetilde{O}(T^{-1/2})$ and we provably converge rapidly to
the Pareto frontier. We note that the bound for the single-objective case can be
recovered by setting $k = 1$ and matches the bounds of previous work
\cite{russo2014learning, srinivasgaussian}. Furthermore, since the single-objective bounds are shown to be tight up to poly-logarithmic factors \cite{scarlett2017lower}, our regret bounds are thus also tight for our dependence on $T$.  Our main theorem can be stated as
follows:

\begin{theorem}[Convergence of Bayesian Optimization with Hypervolume Scalarization: Informal Restatement of Theorem~\ref{thm:hyper_regret}]
The cumulative hypervolume regret for using random hypervolume scalarization with UCB or TS after $T$ observations is upper bounded as
$$\sum_{t=1}^T (\HV_z(Y^* ) - \HV_z(Y_t)) \leq O(k^2 n^{1/2} [ \gamma_T T\ln(T)]^{1/2})$$
Furthermore, $\HV_z(Y_T) \geq \HV_z(Y^*) - \epsilon_T$, where $\epsilon_T = O(k^2 n^{1/2} [ \gamma_T \ln(T)/T]^{1/2})$.
\end{theorem}

Lastly, we show that if the single-objective optimization procedure admits good convergence properties, then we can solve multi-objective optimization by simply applying the single-objective procedure on sufficiently large number of randomly chosen hypervolume scalarizations. Furthermore, this method is provably correct as we can deduce hypervolume error bounds via concentration properties. Specifically, by using the single-objective procedure for $T$ iterations on $l = O(1/\epsilon^k)$ different randomly chosen scalarizations, we can bound the hypervolume error by $\epsilon$ after the total number of observations of $T \cdot l = O(T/\epsilon^k)$. Therefore, we present a novel generic framework that extends any single-objective convergence bound into a convergence bound for the multi-objective case that provably demonstrates convergence to the Pareto frontier. Note that these bounds hold for any algorithm but are not tight and can likely be improved, unlike those presented above in the Bayesian optimization setting.

\begin{theorem}[Convergence of General Optimization with Hypervolume Scalarization: Informal Restatement of Theorem~\ref{thm:hyper_regret_general}]
Let $\alg$ be any single-objective optimization algorithm with the guarantee that it converges to within $\epsilon_T$ of the maximum after $T$ iterations and observations. 

Then, running $\alg$ with random hypervolume scalarization converges to the Pareto frontier and admits an hypervolume error of $O(\epsilon_T)$ after $O(T/\epsilon_T^k)$ observations. 
\end{theorem}
 
 We empirically validate our theoretical contributions by running our multiobjective algorithms with hypervolume scalarizations on the Black-Box Optimization Benchmark (BBOB) functions, which can be used for bi-objective optimization problems \cite{tuvsar2016coco}. We see that our multi-objective Bayesian optimization algorithms, which admit strong regret bounds, consistently outperforms the multi-objective evolutionary algorithms. Furthermore, we observe the superior performance of the hypervolume scalarization functions over other scalarizations, although that difference is less pronounced when the Pareto frontier is even somewhat convex.
 
We summarize our contributions as follows:
\begin{itemize}
    \item Introduction of new hypervolume scalarization function, with connections to hypervolume indicator and provable guarantees.
    \item Development of simple algorithm for approximating hypervolume indicator with hypervolume scalarizations that is accurate due to good smoothness and concentration properties.
    \item Novel hypervolume regret bounds for Bayesian optimization with UCB or TS when using hypervolume scalarization.
    \item Derivation of convergence hypervolume error bounds for any single-objective optimization procedure when using hypervolume scalarization.
\end{itemize}
\section{Preliminaries}
We first define a few notations for the rest of the paper. For two vectors $x, y$, we define $x \leq y$ and similarly all other comparisons/operators element-wise and $\| \cdot \|$ is the Euclidean norm unless specified otherwise. We define $a + b := a + b \mathbf{1}$ for $a \in \R^k$ and $b \in \R$. For a function $f(x) : \R^n \to \R$, we say $f$ is $L$-Lipschitz if $|f(x) - f(x')| \leq L \|x  - x'\|_1$.

Let $\X$ be a compact subset of $\R^n$ and let $F(x) : \R^n \to \R^k$ be our multi-objective function with $k$ objective functions $f_i$. For two points $x_1, x_2 \in \X$, we say that $x_1$ is {\it Pareto-dominated} by $x_2$ if $f_i(x_1) \leq f_i(x_2)$ for all $i$ and there exists $j$ such that $f_j(x_1) < f_j(x_2)$. A point is {\it Pareto-optimal} in $\X$ if no point in $\X$ dominates it. Let $X^*$ denote the set of Pareto-optimal points in $\X$. A {\it Pareto set} for input points $X = \set{x_1,.., x_m}$ is the set of points in $\R^k$  is $\set{F(x) | x \in X}$, denoted as simply $F(X)$. The {\it Pareto frontier} $Y^* := \set{F(x) | x \in X^*}$ is the Pareto set of $X^*$.

Our main progress indicators are given by the hypervolume indicator. For $S \subseteq \R^k$ compact, let $\operatorname{vol}(S)$ be the hypervolume of $S$.

\begin{definition}
For $Y \subseteq \R^k$, we define the (dominated) {\bf hypervolume indicator} of $Y$ with respect to reference point $z$ as:
$$ \HV_z(Y) = \operatorname{vol}(\set{ x \,  |  \, x \geq z,  x \text{ is dominated by some } y \in Y})$$
\end{definition}

Therefore, for a finite set $Y$, $\HV_z(Y)$ can be viewed as the hypervolume of
the union of the dominated hyper-rectangles for each point $y_i \geq z$ that has
one corner at $z$ and the other corner at $y_i$. Note that our definition also
holds for non-finite set as a limiting integral in the Lesbesgue measure. We let
$\Sphere^{k-1} = \set{ y \in \R^{k} \, |\, \|y \| = 1, y \geq 0 }$ and let $y \sim
\Sphere^{k-1}$ denote that $y$ is drawn uniformly on $\Sphere^{k-1}$.

\subsection{Scalarization with Bayesian Optimization}

Bayesian optimization uses a probabilistic model to fit to the blackbox
function. Gaussian processes (GP) are a standard way to model distributions over
functions and are commonly used to derive good regret bounds
\cite{williams2006gaussian}. We will begin with a brief review of GPs and their
role in Bayesian Optimization.

A Gaussian process, $\GP(\mu,\kappa)$, is a distribution over functions. In a
GP, the similarity between points $x_i, x_j$ are determined by the kernel
function $\kappa(x_i, x_j)$ and for some finite set of points $X =
\set{x_1,...,x_m} \in \X$, the distribution of $f(x_1),..., f(x_n)$ over $f$ is
modeled as multivariate Gaussian whose covariance matrix is $\Sigma_{ij} =
\kappa(x_i, x_j)$ and mean is given by $\mu_i = \mu(x_i)$. Examples of popular
kernels are the squared exponential, $\kappa(x_i, x_j) = \exp(- \gamma \| x_i -
x_j\|^2)$, and the Mat\'ern kernel. The mean function, prior to receiving data, is
often assumed to be zero.

When datapoints for values of $y_i = f(x_i) + \epsilon_i$ are received  with $\epsilon_i \sim \N(0, \sigma^2)$, the GP is then updated into a posterior distribution over functions that attempts to fit the datapoints values. This is done by applying conditioning and because all relevant distribution are Gaussian, the resulting distribution is still a GP. This posterior GP induces an unique mean function $\mu(x)$ and standard deviation function $\sigma(x)$ given by the standard formulas:
$$\mu(x) = \kappa(x, X)^\top (\Sigma + \sigma^2 I)^{-1} y $$
$$\sigma(x) = \kappa(x,x) - \kappa(x, X)^\top [\Sigma + \sigma^2 I]^{-1} \kappa(x,X) $$
where $\kappa(x, X)$ denotes the vector with entries as $\kappa(x, x_i)$.

Bayesian optimization then uses the probabilistic model to optimize for the best
inputs for the blackbox function. Since we have a distribution over functions,
the optimization is done via an \emph{acquisition function}, $a(x) : \R^n \to
\R$, that assigns a scalar value to each point based on the aforementioned
distribution.  A popular and natural method is Thompson Sampling (TS), which
draws a random sample from the posterior $f \sim \GP(\mu, \kappa)$, and simply
uses that as a non-deterministic acquisition function $a(x) = f(x)$
\cite{thompson1933likelihood}. Another popular method is the Upper Confidence
Bound (UCB) acquisition that combines the mean and variance function into $a(x)
= \mu(x) + \sqrt{\beta} \sigma(x)$, where $\beta$ is a function of $m$ and $n$
(e.g., $\beta = \Theta(n \ln(m))$ in~\cite{srinivasgaussian}). In both cases,
the acquisition $A(x)$ is maximized as well to suggest the next point to
explore.

For multi-objective Bayesian optimization, we assume that all objectives $f_i$
are samples from known GP priors $\GP(0, \kappa_i)$ with input domain
$\X$. Therefore, we will have a separate GP and acquisition function $A_i(x)$
for each $f_i$ and let the acquisition vector be $A(x) = (a_1(x), ...,
a_k(x))$. It then becomes unclear if there is a single acquisition function that
can be used to explore the Pareto frontier well in this setting. Scalarization
is a well-known means to combine multiple single-objective acquisition functions
for this task. Given a distribution over scalarization functions $s_\lambda(y)$
for $y \in \R^k$, we simply draw a random scalarization function and optimize
according to $s_\lambda(A(x))$. The full algorithm is outline in
Algorithm~\ref{alg:scalarization}.

Examples of scalarizations are the simple linear scalarization $s_\lambda(y)= \sum_i \lambda_i y_i$ and the Chebyshev scalarization $s_\lambda(y) = \min_i \lambda_i (y_i - z_i)$ for some input distribution $\D_\lambda$. However, these scalarizations lack provable guarantees and it is in fact known that the linear scalarization can only provide solutions on the convex part of the Pareto frontier \cite{boyd2004}. Furthermore, the Chebyshev scalarization is criticized for lacking diversity and uniformity in the Pareto frontier \cite{das1998normal}.
\begin{algorithm}[tb]
   \caption{Scalarization for Multi-Objective Bayesian Optimization}
   \label{alg:scalarization}
   \SetKwInOut{Input}{Input}
   \Input{$F: \R^\dim \to \R^k$: multi-objective function , $T \in \Z_+$: number of iterations, \\$A(x): \R^n \to \R^k$ acquisition function, $\D_\lambda$: distribution to sample $\lambda$ for known $s_\lambda(x)$\\}
   \For{$t=1$ {\bfseries to} $T$}{
   \textbf{Sample: } Draw independently $\lambda_t \sim \D_\lambda$.\\
   \textbf{Optimize: }$x_t = \arg\max_{x \in \X} s_{\lambda_t}(A(x))$ where $A = (A_1,.., A_k)$ is evaluated with $\GP^{(t-1)}$\\
    \textbf{Evaluate: } $y_t = F(x_t)$\\
    \textbf{Update:} Incorporate $(x_t, y_t)$ into $\GP^{(t-1)}$ to obtain $\GP^{(t)}$, where the posterior update is done independently for each $f_i$ and $ \GP(\mu_i,\kappa_i)$
   }
   \Return{$\set{x_t}_{t=1}^T$}
\end{algorithm}

\section{Hypervolume Scalarization} \label{sec:scalarization}

In this section, we introduce our novel hypervolume scalarization function and demonstrate that the expectated scalarization value under a certain distribution of weights will give the dominated hypervolume, up to a constant factor difference. Our proof technique relies on a volume integration argument in spherical coordinates and exploits unique properties of the dominated volume. Later, we prove concentration of the empirical mean, allowing for an accurate estimator of the dominated hypervolume.

\begin{lemma}[Hypervolume as Scalarization]
\label{lem:hypervolume}
Let $Y = \set{y_1,..., y_m}$ be a set of $m$ points in $\R^k$. Then, the hypervolume of $Y$ with respect to a reference point $z$ is given by:
$$\HV_z(Y) = c_k \E_{\lambda \sim \Sphere^{k-1}} \left [ \max_{y\in Y} s_\lambda(y - z) \right ]$$
where $s_\lambda(y) = \underset{i}{\min} \,(\max(0,  y_i/\lambda_i))^k$ and $c_k = \frac{\pi^{k/2}}{2^k \ \Gamma(k/2+1)} $ is a dimension-independent constant.
\end{lemma}

Intuitvely, this lemma says that although the maximization of any specific scalarization will bias the optimization to a certain point on the Pareto frontier, there is a specific combination of scalarizations with different weights such that the sum of the maximization of these scalarizations over $Y$ will give the hypervolume of $Y$. It is important for the maximization to be inside the expectation, implying that maximizing various randomized single-objective scalarizations could serve to also maximize the hypervolume, which is our ultimate objective. In fact, it is easy to see that the hypervolume of a set cannot be written solely as a maximization over the expectation of any scalarizations. 

Furthermore, note that our scalarization function, similar to the Chebyshev scalarization, is a minimum over coordinates $s_\lambda(y) = \underset{i}{\min} \,(\max(0,  y_i/\lambda_i))^k$.  Intuitively, this minimization captures the notion that Pareto dominance is a coordinate-wise optimality criterion, as we can redefine $x_1$ being Pareto dominated by $x_2$ as $\underset{i}{\min} f_i(x_2) - f_i(x_1) \geq 0$.

Now that we can rewrite the hypervolume as a specific expectation of maximization of random scalarizations, we proceed to show that our random estimator has controlled variance and therefore concentrates. Specifically, we show that the hypervolume indicator can be computed efficiently via this integral formulation, providing a simple and fast implementation to give an approximation to the hypervolume indicator. Our argument relies on proving smoothness properties of our hypervolume scalarizations for any $\lambda > 0$ and then applying standard concentration inequalities. We note that it is non-obvious why $s_\lambda(y)$ is smooth, since $s_\lambda(y)$ depends inversely on $\lambda_i$ so when $\lambda_i $ is small, $s_\lambda$ might change very fast. The full proof is in the appendix.

\begin{lemma}[Hypervolume Concentration]
\label{lem:concentration}
Let $Y = \set{y_1,..., y_m} \subseteq \R^k$ and $z$ be a reference point and assume for some $B
\geq 1$, we can bound $y_i \leq z + B$. Then, $s_\lambda(y)$ is $O(B^k
k^{1+k/2})-$Lipschitz for any $\lambda$ and we can independently draw
$\lambda_1, ..., \lambda_s$ such that
$$\left|\frac{1}{c_k}\HV_z(Y) - \frac{1}{s} \sum_j \max_{y \in Y} s_{\lambda_j}(y-z)\right| \leq \epsilon $$
holds with probability at least $1-\delta$ with $s = O(B^{2k}k^k\log(1/\delta)/\epsilon^2)$ samples.
\end{lemma}

\section{Hypervolume Regret Bounds} \label{sec:regret}

In this section, we derive novel hypervolume regret bounds for Bayesian optimization with UCB and TS acquisition functions under the hypervolume scalarizations. Furthermore, we introduce an algorithmic framework to turn any single-objective optimization algorithm to a multi-objective optimization algorithm via scalarizations and show that single-objective convergence guarantees can provably translate into multi-objective convergence bounds.

For a fixed weight $\lambda$, if $X = \set{x_1,...,x_m}$ is our dataset and recall $X^*$ is the maximal input set corresponding to the Pareto frontier $Y^*$, we define the scalarized regret to be
$$r_\lambda(X) = \max_{x \in X^*} s_\lambda(F(x)) - \max_{x \in X} s_\lambda(F(x))$$
The scalarized regret only captures the regret for a specific weighting $\lambda$ and we emphasize that scalarized regret bounds will guarantee convergence to only one point in $X^*$. To prove convergence to the whole frontier, we consider the following {\it Bayes regret} given by $R(X) = \E_{\lambda}\left[ r_\lambda(X)\right]$, which integrates over all $\lambda$ for some distribution $\D_\lambda$.

Note that Bayes regret cannot be minimized by a single point, rather it requires a set of points across the Pareto frontier to achieve a small Bayes regret. Let $X_T = \set{x_1,...,x_T}$ be the set of chosen inputs up to time $T$. Then, we wish to bound $R(X_T)$. We cannot analyze the Bayes regret directly. Rather, we bound it via a slightly surrogate regret measure. Let use define the instantaneous regret at time step $t$ as:
$$ r(x_t, \lambda_t) = \max_{x \in X^*} s_{\lambda_t}(F(x)) - s_{\lambda_t}(F(x_t))$$
The cumulative regret at time step $T$ is then $R_C(T) = \sum_{t=1}^T r(x_t, \lambda_t)$. We will first bound the cumulative regret and then use the cumulative regret to bound the Bayes regret, which is a scaled version of our hypervolume regret when using hypervolume scalarizations.

Regret bounds are derived using a metric known as the {\it maximum information gain} (MIG), which captures an information-theoretic notion of uncertainty reduction of our blackbox function. For any $A \subseteq \X$, we define the random set $y_A = \set{y_a = f(a) + \epsilon_a | a \in A}$ be our noisy evaluations. The reduction in uncertainty of the distribution over functions $f$ induced by the GP by observing $y_A$ is given by the mutual information ${\bf I}(y_A; f) = {\bf H}(f) - {\bf H}(f | y_A) = {\bf H}(y_A) - {\bf H}(y_A | f)$, where ${\bf H}$ denotes the Shannon entropy. The maximum information gain after $T$ observations is defined as :
$$\gamma_T = \max_{A \subseteq \X : |A | = T} {\bf I}(y_A; f) $$

The mutual information can also be explicitly calculated via a useful formula: ${\bf I}(y_A; f) = \frac{1}{2} \log | I + \sigma^{-2} K_A|$, where $K_A$ is the $T$-by-$T$ covariance matrix of dataset A and $|\cdot|$ is the determinant operator. Using the formula, one can derive bounds of $\gamma_T = O((\log T)^{\dim+1})$ for the squared exponential kernel and similar bounds for the Mat\'ern and linear/polynomial kernels for any $A$ \cite{srinivasgaussian}. When the maximum information gain is small, the $\GP$ function distribution induced by the corresponding is relatively easier to model and regret bounds are therefore tighter.

\begin{theorem}[Theorem 1 in \cite{paria2018flexible}]\label{thm:regret}
Let each objective $f_i(x)$ for $x \in [0,1]^n$ follow a Gaussian distribution with marginal variances bounded by 1 and observation noises $\epsilon_i \sim \N(0, \sigma_i^2)$ are independent with $\sigma_i^2 \leq \sigma^2 \leq 1$. Let $\gamma_{T, k} \leq \gamma_T$, where $\gamma_{T, k}$ is the MIG for the k-th objective. Running Algorithm~\ref{alg:scalarization} with $L$-Lipschitz scalarizations on either UCB or TS acquisition function produces an expected cumulative regret after $T$ steps that is bounded by:
$$\E_[R_C(T)] = O(L k n^{1/2} [  \gamma_T T \ln(T)]^{1/2}) $$

where the expectation is over choice of $\lambda_t$ and $\GP$ measure.
\end{theorem}

\begin{theorem}
\label{thm:hyper_regret}
Assume the conditions in Theorem~\ref{thm:regret} holds and let $F(\X) \subseteq [0, 2/5]^{k}$ and $z \geq 0$ and $Y_t = F(X_t)$. Running Algorithm~\ref{alg:scalarization} with hypervolume scalarizations with reference point $z$ and $D_\lambda = \Sphere^{k-1}$ on either UCB or TS acqusition function produces hypervolume regret after $T$ observations that is bounded by:
$$\sum_{t=1}^T (\HV_z(Y^* ) - \HV_z(Y_t)) \leq O(k^2 n^{1/2} [ \gamma_T T\ln(T)]^{1/2})$$
Furthermore, $\HV_z(Y_T) \geq \HV_z(Y^*) - \epsilon_T$, where $\epsilon_T = O(k^2 n^{1/2} [ \gamma_T \ln(T)/T]^{1/2})$.
\end{theorem}

\begin{proof}
WLOG, let $z = 0$. From Lemma~\ref{lem:hypervolume}, we see that for our hypervolume scalarization, we can rewrite our regret using the Bayes regret.
\begin{align*}
    \HV_z&(Y^* ) - \HV_z(Y_t) \\
    &= c_k \E_{\lambda\sim \Sphere^{k-1}}\left[ \max_{x \in X^*} s_\lambda(F(x)) - \max_{x \in X_t} s_\lambda(F(x))\right] \\
    &= c_k\E_{\lambda\sim \Sphere^{k-1}}\left[r_\lambda(X_t)\right] \\
    &= c_k R(X_t)\\
\end{align*}
Therefore, our Bayes regret is exactly a scaled version of our hypervolume regret with $c_k = \pi^{k/2}/(2^k \Gamma(k/2 +1)) \leq (\pi^{k/2}e^{k/2})/(2^k (k/2+1)^{k/2})$ with standard bounds on $\Gamma$. Next, we use the instantaneous regret to bound Bayes risk and note that $\lambda_t \sim \Sphere^{k-1}$.
\begin{align*}
    R(X_t) &=  \E_{\lambda\sim \Sphere^{k-1}}\left[ \max_{x \in X^*} s_\lambda(F(x)) - \max_{x \in X_t} s_\lambda(F(x))\right] \\
    &\leq   \E_{\lambda\sim \Sphere^{k-1}}\left[ \max_{x \in X^*} s_\lambda(F(x)) -  s_\lambda(F(x_t))\right] \\
    &\leq  \E_{\lambda}\left[ r(x_t, \lambda)\right]
\end{align*}
Therefore, by Theorem~\ref{thm:regret}, we conclude that we can bound the Bayes and hypervolume regret by observing the following relations:
\begin{align*}
    \sum_{i=1}^T R(X_t) &\leq \E_{\lambda_1,...,\lambda_t}\left[\sum_{i=1}^T r(x_t, \lambda_t)\right] \\ 
    &\leq  \E[R_C(T)] \\
    &= O(Lk n^{1/2} [  \gamma_T T \ln(T)]^{1/2})\\
\end{align*}
Lastly, by Lemma~\ref{lem:concentration}, we see that $L \leq 2^{-k}k^{1+k/2}$ and note that $c_k L \leq k(\pi^{k/2}e^{k/2})/(5^k) \leq k$. Together, we finally conclude that
$$\sum_{t=1}^T (\HV_z(Y^* ) - \HV_z(Y_t)) \leq O(k^2 n^{1/2} [ \gamma_T T\ln(T)]^{1/2})$$
For the final claim, by definition of $\HV_z$ and our equivalence between $\HV_z$ and $R$, we conclude that $R(X_t)$ must be monotonically decreasing. Our final claim follows from monotonicity and simple algebra.
\end{proof}

We note that our regret bounds hold for classical Bayesian optimization procedures that are widely used with no artificial modifications. Furthermore, we can generalize our results to show that for any single-objective optimization procedure, one can use hypervolume scalarizations to convert the procedure into a multi-objective optimization procedure via a natural extension (Algorithm~\ref{alg:scalarization_general}). More remarkably, if the single-objective optimization procedure admits convergence bounds, we can immediately derive hypervolume convergence bounds by appealing to our previous connections between scalarization and hypervolume via the following theorem. Practically, this implies that any currently used single-objective optimization algorithms can be easily generalized to the multi-objective setting with minimal effort and enjoy provable convergence bounds. The full proof is in the appendix.

\begin{algorithm}[tb]
   \caption{Scalarization with General Single-Objective Optimization}
   \label{alg:scalarization_general}
   \SetKwInOut{Input}{Input}
   \Input{$F: \R^\dim \to \R^k$: multi-objective function , $T \in \Z_+$: number of iterations, $\alg: $ single-objective optimization algorithm, $\D_\lambda$: distribution to sample $\lambda$ for known $s_\lambda(x)$, $l \in \Z_+$: number of scalarization samples}
   \For{$i=1$ {\bfseries to} $l$}{
   \textbf{Sample: } Draw independently $\lambda_i \sim \D_\lambda$.\\
   \textbf{Optimize: } Run $\alg$ on single-objective function $s_{\lambda_i}(F(x))$ for $T$ iterations to obtain points $\set{x_{\lambda_i}^{(1)}, x_{\lambda_i}^{(2)},...,x_{\lambda_i}^{(T)}}$
   }
   \Return{ $\underset{i}{\bigcup} \set{x_{\lambda_i}^{(t)}}_{t=1}^T$ }
\end{algorithm}

\begin{theorem}
\label{thm:hyper_regret_general}
Let $F(\X) \subseteq [0, B]^{k}$ and $z \geq 0$ and let $\alg$ be any single-objective maximization algorithm on objective function $g(x)$ that guarantees that after $T$ iterations, it returns $x_T$ such that  $g(x_T) \geq  g(x^*)  - \epsilon_T$, where $x^*$ is the optima. Then, running Algorithm~\ref{alg:scalarization_general} on hypervolume scalarizations with reference point $z$ and $D_\lambda = \Sphere^{k-1}$ with $l = \widetilde{O}((2B)^{k^2} k^{k+1}/(c_k^{k-1}\epsilon_T^{k+1}))$ converges to the Pareto frontier and after $l\cdot T$ observations, we have
$$\HV_z(Y_T) \geq \HV_z(Y^*) - 5 c_k \epsilon_T$$
\end{theorem}

\section{Experiments}

We empirically demonstrate the utility of hypervolume scalarizations by running our proposed multiobjective algorithms on the Black-Box Optimization Benchmark (BBOB) functions, which can be paired up into multiple bi-objective optimization problems \cite{tuvsar2016coco}. To emphasize the utility of a new theoretically sound scalarization, we focus on scalarization-related algorithms and comparisons were not made to the vast array of diverse multi-objective blackbox optimization used in certain practical settings. Therefore, our goal is therefore to compare scalarizations on commonly used optimization algorithms, such as UCB and evolutionary strategies. Furthermore, we note that scalarized algorithms have very fast practical runtimes and are especially relevant in settings with a low computational budget. 

Our objectives are given by BBOB functions, which are usually non-negative and are minimized. The input space is always a compact hypercube $[-5,5]^n$ and the global minima is often at the origin. For bi-objective optimization, given two different BBOB functions $f_1, f_2$, we attempt to maximize the hypervolume spanned by $(-f_1(x_i), -f_2(x_i))$ over choices of inputs $x_i$ with respect to the reference point $(-5, -5)$. Therefore, all relevant output points are contained in the square between $(-5, -5)$ and $(0,0)$, giving a maximum hypervolume of 25. Because BBOB functions can drastically different ranges, we first normalize the function by a measure of standard deviation computed by taking the empirical variance over a determinsitic set of $30$ different inputs. We also apply large random shifts/rotations as well as allow for adding moderate random observation noise to the objective function.
\begin{figure}[H]
\vskip 0.2in
\begin{center}
\centerline{\includegraphics[width=3in]{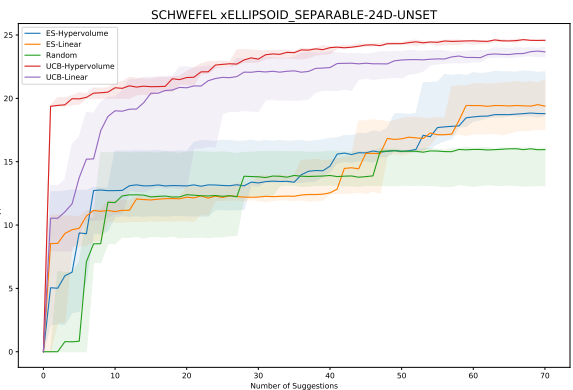}}
\centerline{\includegraphics[width=3.1in]{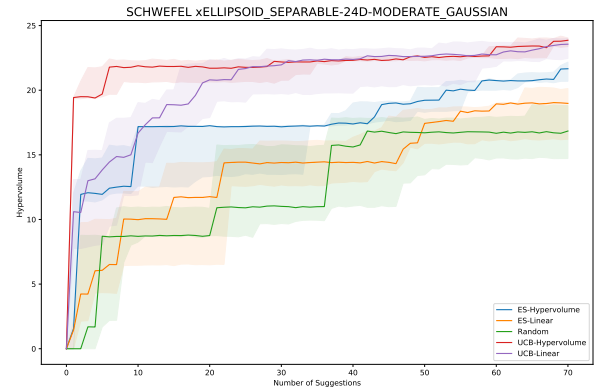}}
\caption{Dominated hypervolume plot of bi-objective optimization of SCHWEFEL and ELLIPSOID BBOB functions in 24D with no and Gaussian observation noise. Notice that Hypervolume scalarization is slightly better than Linear scalarization in both UCB and ES algorithms and this is more stark in the noiseless case.}
\label{fig:hypervolume_plot3}
\end{center}
\vskip -0.2in
\end{figure}

We run each of our algorithms in dimensions $n = 8, 16, 24$ and optimize for $70$ iterations with $5$ repeats. Our algorithms are the Random algorithm, UCB algorithm, and Evolutionary Strategy (ES). Our scalarizations include the linear and hypervolume scalarization with the weight distribution $D_\lambda$ as $\Sphere^1$. Note that for brevity, we do not include the Chebyshev scalarization because it is almost a monotonic transformation of the hypervolume scalarization with a different weight distribution. We run the UCB algorithm via an implementation of Algorithm~\ref{alg:scalarization} with a constant standard deviation multiplier of $1.8$ and a standard Mat\'ern kernel, while we run the ES algorithms using Algorthm~\ref{alg:scalarization_general} with $T  = 1$ and $l = 70$ by relying on a well-known single-objective evolutionary strategy known as Eagle \cite{yang2010eagle}.

From our results, there is a clear trend that UCB algorithms outperform both ES and Random algorithms in all cases, empirically affirming the utility of Bayesian optimization and its strong theoretical regret bounds. We believe that this is because evolutionary strategies tend to follow local descent procedures and get stuck at local minimas, therefore implicitly having less explorative mechanisms. Meanwhile, Bayesian optimization inherently promotes exploration via usage of its standard deviation estimates.

Furthermore, we see that the hypervolume scalarization slightly outperforms the linear scalarization, with UCB-Hypervolume being a clear winner in certain cases, even with Gaussian observation noise (see Fig~\ref{fig:hypervolume_plot1}). The superior performance of the hypervolume scalarization is seen in both UCB and ES algorithms (see Fig~\ref{fig:hypervolume_plot3}), and using the hypervolume scalarization is often never worse than using the linear scalarization.  We note that the difference is more stark when there is no noise added, allowing for less variance when calculating the scalarization. Also, we note that when the Pareto frontier is almost convex, the difference in performance becomes hard to observe. As seen from the Pareto plot in Fig~\ref{fig:hypervolume_plot1}, we see that although UCB-Hypervolume does produce a better Pareto frontier than UCB-Linear, the convex nature of the Pareto frontier allows linear scalarizations to perform decently. A complete profile of the plots are given in the appendix.

\begin{figure}[h]
\vskip 0.2in
\begin{center}
\centerline{\includegraphics[width=3in]{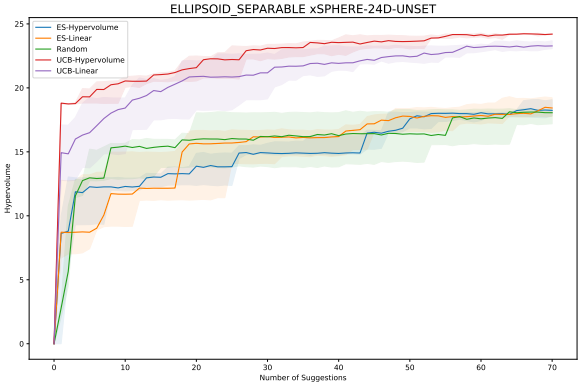}}
\centerline{\includegraphics[width=3.1in]{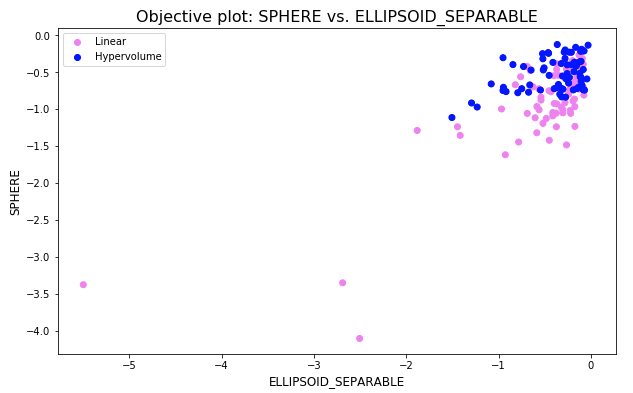}}
\caption{Dominated hypervolume and Pareto plot of bi-objective optimization of ELLIPSOID and SPHERE BBOB functions in 24D with no (UNSET) observation noise. Notice that the UCB algorithms outperform the ES algorithms and the Hypervolume scalarization is slightly better. Also, when comparing UCB-Hypervolume (blue points) and UCB-Linear (pink points) in the Pareto plot, the Hypervolume scalarization is slightly better and never suggests points that are significantly sub-optimal, although the difference is not very stark because of the convex nature of the Pareto frontier.}
\label{fig:hypervolume_plot1}
\end{center}
\vskip -0.2in
\end{figure}

\section{Conclusion}
We introduced the hypervolume scalarization functions and utilized its connection to the hypervolume indicator to derive hypervolume regret bounds for many classes of multi-objective optimization algorithms that rely on single-objective subroutines. Even though our scalarization function enjoys smoothness and concentration bounds, we note that possible improvements to the scalarizations can be made for variance reduction. We believe that a lower-variance and better-concentrated scalarization can be constructed that also can similar provable hypervolume error guarantees. Furthermore, it is conceivable that variance reduction techniques during the optimization process can be applied to achieve better concentration and convergence.

Furthermore, our regret bounds can likely be improved, especially the general regret bound for any single-objective optimization procedure given by Theorem~\ref{thm:hyper_regret_general}. We believe the exponential dependence on $k^2$ can be improved, as well as removing the extra $l$ factor completely. However, to achieve these better convergence bounds, it will most likely require algorithmic changes. Lastly, our experiments could be improved by concocting a specific optimization with a concave Pareto frontier, looking at optimization with more than two objectives, or considering a more diverse set of algorithms. Understanding and quantifying the full impact of using random scalarizations to generalize single-objective algorithms to multi-objective algorithms is an open problem.

\newpage

\bibliographystyle{alpha}
\bibliography{main}

\appendix 
\newpage
\section{Missing Proofs}

\begin{proof}[Proof of Lemma~\ref{lem:hypervolume}(Hypervolume as Scalarization)]
  Without loss of generality, let $z = 0$ be the origin and consider computing
  the volume of a rectangle with corners at the origin and at $y = (y_1,...,y_k)
  \geq 0$, but doing so in polar coordinates centered at $z$. Given a direction
  $v \in \Sphere^{k-1}$ with $v_i \ge 0$ for all $i$, and $\|v\| = 1$, lets
  suppose a ray in the direction of $v$ exits the rectangle at a point $p = cv$
  so that $\| p\| = c$.  We claim that $\| p\| = \min_i (y_i / v_i)$. Note that this claim holds in any norm; however for the eventual dominated hypervolume calculation to hold, we use the $\ell_2$ norm.

  Note that by definition of $p$ we have $p_j \le y_j$ for all $j$ and
  there must exist $i$ such that $p_i = y_i$.
  Also $p_j/v_j = c = p_i/v_i$ for any $i, j$ since $p = cv$.
  It follows that $c \le y_j / v_j$ for all $j$ and $c = y_i / v_i$, which
  proves $\| p\| = \min_i (y_i / v_i)$.



Integrating in polar coordinates, we can approximate an volume via radial slivers of the circle, which for a radius $r$ sweeping through angles $d\theta$ have an volume proportional to $r^k d\theta$. Hence, the volume of the  of the rectangle is
$$\operatorname{vol}(R) = c_k \int_{v \in \Sphere^{k-1}} \min_i \left( \frac{y_i}{v_i}\right)^k d\theta(v) $$
under a uniform measure $\theta$, where the $c_k$ is a constant that depends only on the dimension.

So far we assumed $y \geq 0$. Now, if any $y_i < 0$, then our total dominated hypervolume is zero and $\min_i \frac{y_i}{v_i} < 0$. So, by changing our scalarization slightly, we can account for any $y$ and the volume of the rectangle with respect to the origin is given by:
$$ \operatorname{vol}(R) = c_k \int_{v \in \Sphere^{k-1}} \min_i \left(\max(0, y_i/ v_i)\right)^k d\theta(v) $$
By definition of dominated hypervolume, note that $\HV_z(Y) = \operatorname{vol}(S)$ where
$$S = \set{ x \, | \, x \geq z,  x \text{ is dominated by some } y \in Y}.$$
Since $S$ is simply the union of rectangles at $y_1,...,y_m$ and note that wherever $p$ exists $S$, the length of $p$ is the maximal over all rectangles and so
$$\|p \| = \max_{y \in Y} \min_{i}  \max(0, y_i/v_i).$$
Repeating the argument gives:
$$\operatorname{vol}(S) = c_k \int_{v \in \Sphere^{k-1}} \max_{y \in Y} \left [\min_i \left(\max(0, y_i/ v_i)\right)^k \right] d\theta(v) $$
To calculate $c_k$, we simply evaluate the hypervolume of the $k$-dimensional ball in the positive orthant. If the ball has radius and is centered at the origin. In this case we get $\int_{v \in \Sphere^{k-1}} r^k d\mu(v) = r^k$ and $c_k \int_{v \in \Sphere^{k-1}} r^k d\mu(v) = V_k(r) / 2^k$ where $V_k(r)$ is defined as the volume of the $k$-dimensional ball of radius $r$, which is
$\pi^{k/2}r^k \ / \ \Gamma(k/2+1)$.
The formula for $c_k$ then follows from some basic algebra.
\end{proof}

\begin{proof}[Proof of Lemma~\ref{lem:concentration} (Hypervolume Concentration)]
Recall $s_\lambda(y) = \underset{i}{\min} (\max(0, y_i/\lambda_i))^k$ and $\|\lambda\| = 1$. Note that there must exists $i$ such that $\lambda_i \geq k^{-1/2}$ and therefore $\paren{\max(0, y_{i^*}/\lambda_{i^*})}^k \leq (Bk^{1/2})^k$ where $i^*$ is the index that minimizes $\max(0, y_i/\lambda_i)^k$. Note that the gradient of $s_\lambda(y)$, if non-zero, is $ky_{i^*}^{k-1}/\lambda_{i^*}^k$ and therefore, the Lipschitz constant is bounded by $k (Bk^{1/2})^k = B^k k^{1+k/2}$.

Since $0 \leq s_{\lambda}(y - z) \leq B^k k^{k/2}$ for any $\lambda$ and $y$, we conclude by standard Chernoff bounds that if weight vectors $\lambda_j$ are independent samples,

$$ \Pr\paren{ |\E_{\lambda \sim \Sphere^{k-1}}[\max_{y\in Y} s_\lambda (y-z)] - \frac{1}{s} \sum_j \max_{y \in Y} s_{\lambda_j}(y-z)| \geq \epsilon }$$ $$\leq 2\exp(- 2s\epsilon^2/(B^{2k}k^k)) $$

Therefore, choosing $s = O(B^{2k}k^k\log(1/\delta)/\epsilon^2)$ samples from $\Sphere^{k-1}$ bounds the failure probability by $\delta$ and using Lemma~\ref{lem:hypervolume}, our result follows.
\end{proof}

\begin{proof}[Proof of Theorem~\ref{thm:hyper_regret_general} (General Regret Bounds)]
WLOG, let $z = 0$. Let $X_T = \underset{i}{\bigcup} \set{x_{\lambda_i}^{(t)}}_{t=1}^T$. Then, for $\lambda_1,...,\lambda_l$, by the guarantees of $\alg$, we deduce that

$$\frac{1}{l}\sum_{i=1}^l \left[ \max_{x\in\X}s_{\lambda_i}(F(x)) - \max_{x \in X_T} s_{\lambda_i}(F(x)) \right]\leq \epsilon_T$$

By Lemma~\ref{lem:concentration}, we see that for the Pareto frontier $Y^*$, we have concentration to the desired hypervolume:
$$\left|\frac{1}{c_k}\HV_z(Y^*) - \frac{1}{l} \sum_i \max_{x\in\X}  s_{\lambda_i}(x)\right| \leq \epsilon $$

when $l = O(B^{2k}k^k\log(1/\delta)/\epsilon^2)$ with probability $1-\delta$. We would like to apply the same lemma to also show that our empirical estimate is close to $\HV_z(X_T)$. However, since $X_T$ depends on $\lambda_i$, this requires a union bound and we proceed with a $\epsilon$-net argument.

We assume that $F(\X) \subseteq [0, B]^k$ and let us divide the hypercube into a grid with spacing $\Delta$. Then, there are $O((B/\Delta)^k)$ lattice points on the grid. Consider the Pareto frontier of any set of points, call it $S$. Out of the $(B/\Delta)^k$ small hypercubes of volume $\Delta^k$ in grid, note that by the monotonicity property of the frontier, $S$ intersects at most $(2B/\Delta)^{k-1}$ small hypercubes. 

Therefore, we can find a set $S_u$ consisting of at most $(2B/\Delta)^{k-1}$ lattice points on the grid such that $ |\HV_z(S) - \HV_z(S_u)| \leq (2B)^k\Delta$. This can be done by simply looking at each small hypercube that has intersection with $S$ and choosing the lattice point that increases the dominated hypervolume. Since each small hypercube has volume $\Delta^k$ and there are at most $(2B/\Delta)^{k-1}$ hypercubes, the total hypervolume increased is at most $2^k \Delta$. 

Finally, there are at most $(B/\Delta)^{k(2B/\Delta)^{k-1}}$ choices of $S_u$, so to apply a union bound over all possible sets $S_u$, we simply choose $l = O(B^{2k} k^{k+1} (2B/\Delta)^{k-1}\log(B/\Delta)/\epsilon^2)$ so that for any possible $S_u$, we use Lemma~\ref{lem:concentration} to deduce that with high probability,

$$\left|\frac{1}{c_k}\HV_z(S_u) - \frac{1}{l} \sum_i \max_{x\in S_u}  s_{\lambda_i}(x)\right| \leq \epsilon $$

Since $\HV_z(S)$ is close to $\HV_z(S_u)$, we conclude that for any $S$, 

$$\left|\frac{1}{c_k}\HV_z(S) - \frac{1}{l} \sum_i \max_{x\in S}  s_{\lambda_i}(x)\right| \leq \epsilon + \frac{(2B)^k\Delta}{c_k} $$

Together, we conclude that 

$$ |\HV_z(F(X_T)) - \HV_z(Y^*)| \leq c_k\epsilon_T + 2c_k \epsilon + (2B)^k\Delta$$

By choosing $\epsilon = \epsilon_T$ and $\Delta = c_k\epsilon_T(2B)^{-k}$, we conclude.

\end{proof}

\section{Figures}

\begin{figure*}[ht]
\vskip 0.5in
\begin{center}
\centerline{\includegraphics[width=7.5in]{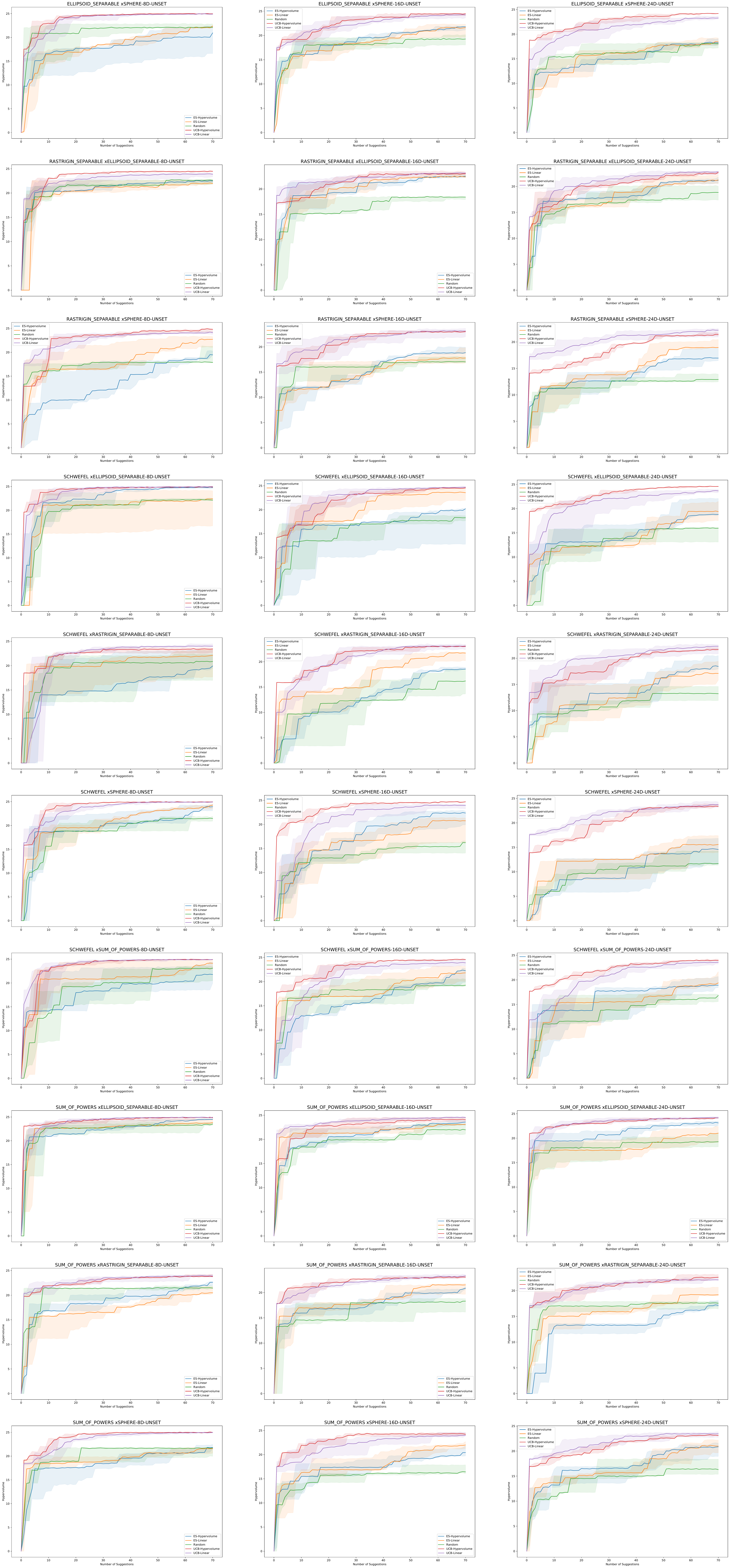}}
\end{center}
\end{figure*}

\begin{figure*}[ht]
\vskip 0.5in
\begin{center}
\centerline{\includegraphics[width=7.5in]{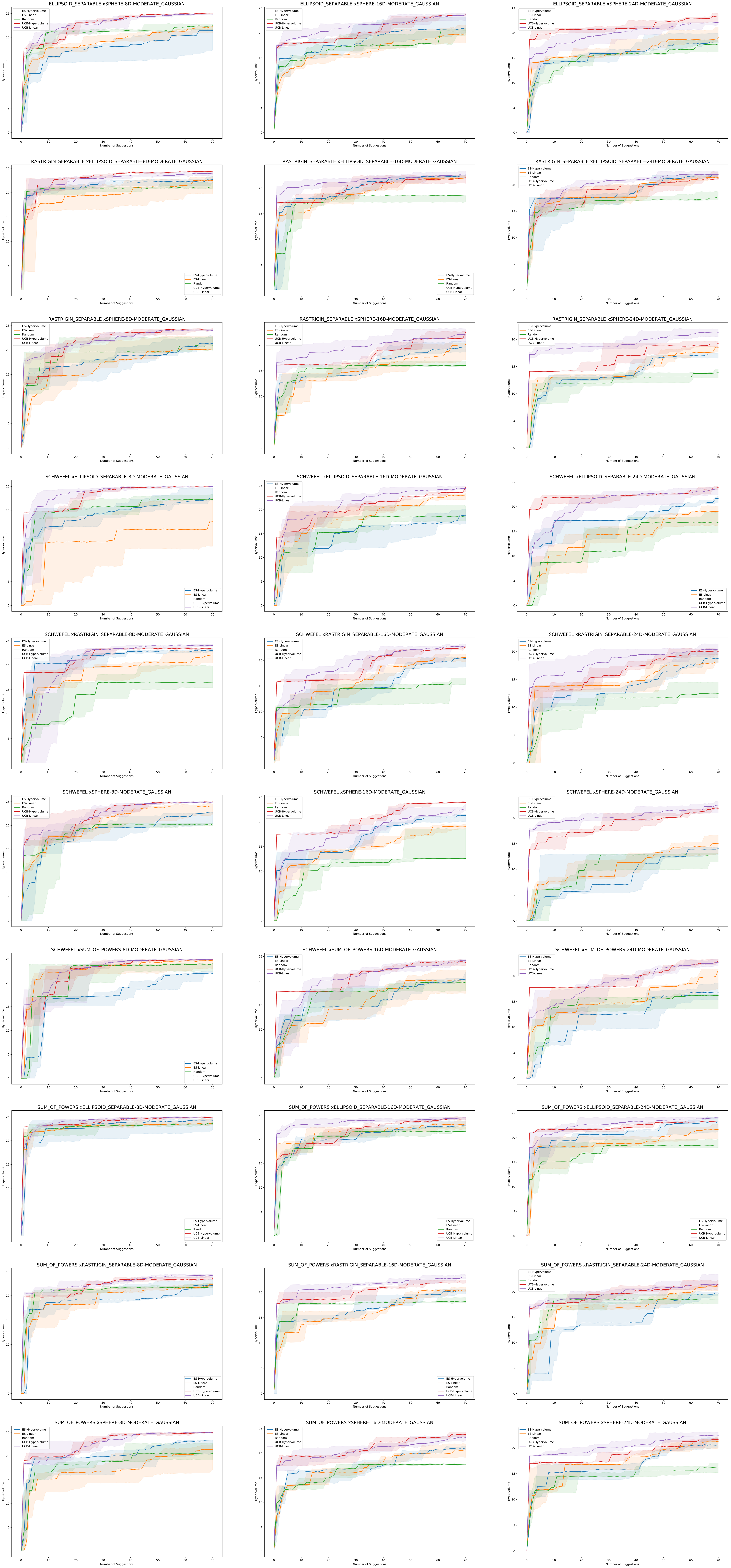}}
\end{center}
\vskip -0.2in
\end{figure*}

\end{document}